\documentclass[11pt]{article}

\usepackage{acl}

\usepackage[utf8]{inputenc}
\usepackage[T1]{fontenc}
\usepackage{times}
\usepackage{microtype}

\usepackage{amsmath}
\usepackage{amssymb}
\usepackage{amsthm}

\usepackage{booktabs}
\usepackage{multirow}
\usepackage{array}

\usepackage{tikz}
\usetikzlibrary{arrows.meta, positioning, calc, fit, decorations.pathreplacing}

\usepackage{listings}

\lstdefinelanguage{Haskell}{
  morekeywords={data,where,let,in,case,of,if,then,else,do,class,instance,type,newtype,deriving},
  sensitive=true,
  morecomment=[l]{--},
  morestring=[b]",
}

\lstdefinelanguage{Rust}{
  morekeywords={pub,struct,impl,fn,let,mut,where,self,const,trait,use,mod,crate,match,Some,None,Vec,in},
  sensitive=true,
  morecomment=[l]{//},
  morecomment=[s]{/*}{*/},
  morestring=[b]",
}

\newtheorem{definition}{Definition}

\newtheorem{proposition}{Proposition}

\theoremstyle{remark}

\newtheorem{example}{Example}

\title{Comonadic Morphophonology: A Compositional Framework\\
for Context-Dependent Morphological Rules in Finnish}

\author{
  Yongseok Jang\thanks{Claude (Anthropic) was used for LaTeX polishing and code-writing assistance; the author takes full responsibility for the content.} \\
  Independent Researcher \\
  \texttt{yongsk0066@gmail.com}
}

\date{}

\begin{document}
\maketitle

\begin{abstract}
Composing finite-state transducers (FSTs) for context-dependent
morphophonological rules---consonant gradation, vowel harmony,
possessive suffix assimilation---leads to multiplicative state
explosion; neural models sidestep the problem but provide no formal
account of the rules themselves.
We present the first framework where each morphophonological rule is a function from a focused local context to a single output segment---the type of a local rule familiar from cellular automata---and where length-changing rules compose as coKleisli arrows of a comonad.
Our central contribution is the \emph{Writer comonad}
($\mathsf{DeletionSet} \times \mathsf{Zipper}$), a new algebraic
construction that restores strict coKleisli compositionality for
such rules: each rule is a coKleisli
arrow, $\mathsf{extend}$ lifts it to a global transformation, and
deletions accumulate as a monoid action rather than requiring
intermediate materialization.
As supporting evidence, thirteen coKleisli arrows provide an alternative
formulation expressing the same morphophonological behaviors that
Omorfi encodes via 874 continuation classes (67:1 reduction at the
rule-representation level), and the same abstraction enables
bidirectional morphology---a MorphGenerator reuses the analysis arrows
for generation.
On UD Finnish-TDT, the system achieves 83.92\% UPOS accuracy with
rule-only disambiguation (94.66\% with an external suffix tagger),
validating the framework as a practical morphological engine.
\end{abstract}

\section{Introduction}
\label{sec:introduction}

A standard FST treats a rule as a string-to-string relation. We instead
treat each rule as a function from a focused local context to a single
output segment---precisely the type of a local rule in cellular
automata. Composing such arrows is not function composition: the output
of one rule has type $\Sigma$, not $\mathsf{Zipper}\langle\Sigma\rangle$.
The coKleisli construction over a Zipper comonad fills exactly this
typing gap.

We show that morphophonological rules---including length-changing
operations such as consonant deletion---are coKleisli arrows of a Writer
comonad, yielding the first framework that handles
length-changing morphophonological rules within a compositional
category-theoretic structure. Thirteen coKleisli arrows provide an
alternative formulation expressing the same morphophonological behaviors
that Omorfi encodes via 874 continuation classes (67:1 reduction at the
rule-representation level via orthogonal composition),
and the same abstraction unifies character-level, morpheme-level, and
sentence-level processing.

Finnish morphophonology motivates the framework.
\emph{Consonant gradation} alternates stem-final consonants between
strong and weak grades depending on syllable structure;
\emph{vowel harmony} requires suffix vowels to agree in backness with
stem vowels;
and \emph{possessive suffix vowel copying} inserts a vowel identical to
the preceding segment at morphological boundaries.
Rule-based systems such as Voikko \citep{pitkanen2006voikko}
and Omorfi \citep{pirinen2015omorfi} encode these alternations
as FST cascades
\citep{koskenniemi1983two,beesley2003finite}, achieving
excellent coverage but limited compositionality: composing two
FSTs yields a product automaton whose state space grows
multiplicatively. Neural approaches
\citep{kanerva2018turku,nguyen2021trankit} achieve high
accuracy but provide no formal account of the underlying
processes. No existing framework treats morphophonological rules
as first-class compositional objects with formal combination laws.

Our goal is not to re-derive known rules, but to provide a
\emph{composition algebra} for rules that FST cascades handle only
through state-space multiplication. The Writer comonad---which tracks
deletions as a monoid action, deferring materialization to pipeline
end---is a new algebraic construction with no prior analogue in
computational morphology.
In contrast to FST $\varepsilon$-transitions, which require
recomposition after each length-changing rule, the Writer comonad
accumulates deletions within the coKleisli pipeline, preserving strict
associativity without intermediate re-materialization.

Morphophonological rules are naturally \emph{coKleisli arrows}: functions
$W\,a \to b$ where $W$ is a comonad providing a context window around a
focused element. The comonad operation $\mathsf{extend}$ lifts such a
local rule into a global transformation, and coKleisli composition
($\mathbin{>\!\!=\!\!>}$) chains rules with associativity and identity
guarantees.
This connection is not accidental: \citet{capobianco2010categorical}
established that cellular automaton local behavior is precisely a
coKleisli map of the exponent comonad, and morphophonological rule
application instantiates exactly this pattern.

Our primary contribution is the Writer comonad
($\mathsf{DeletionSet} \times \mathsf{Zipper}$), which restores strict
coKleisli compositionality for length-changing morphophonological
rules---a problem that FST cascades leave unresolved without
multiplicative state explosion (\S\ref{sec:formalization},
\S\ref{sec:writer-comonad}). As supporting demonstrations:
\begin{itemize}
  \item The same comonadic abstraction operates at multiple linguistic
    levels---character (gradation), morpheme (harmony), and sentence
    (Constraint Grammar (CG) lite disambiguation)---via the uniform
    Zipper/extend mechanism (\S\ref{sec:cg-rules}).

  \item The coKleisli pipeline enables bidirectional morphology: a
    MorphGenerator reuses the same arrows for both analysis and
    generation (\S\ref{sec:pipeline}).
\end{itemize}
We evaluate on UD Finnish-TDT, reporting 83.92\% UPOS with rule-only
disambiguation (94.66\% with a suffix tagger), per-rule latency
microbenchmarks, and a 67:1 reduction at the rule-representation level
(\S\ref{sec:evaluation}).

\section{Background}
\label{sec:background}

\subsection{Finnish Morphophonology}
\label{sec:finnish-morphophonology}

\paragraph{Consonant gradation} (\emph{astevaihtelu}) alternates
stem-final consonants between a \emph{strong grade} and a \emph{weak
grade} \citep{visk2004}. The grade is determined by syllable structure:
open syllables trigger the strong grade, closed syllables the weak
grade. The alternation affects stops and certain consonant clusters
according to the patterns in Table~\ref{tab:gradation-patterns}.

\begin{table}[t]
\centering
\small
\begin{tabular}{@{}cllll@{}}
\toprule
\# & Strong & Weak & Type & Example \\
\midrule
1  & pp & p   & Quant.        & kaappi $\to$ kaapi  \\
2  & tt & t   & Quant.        & matto $\to$ mato    \\
3  & kk & k   & Quant.        & kukka $\to$ kuka    \\
4  & p  & v   & Qual.~(s)  & tupa $\to$ tuva   \\
5  & t  & d   & Qual.~(s)  & katu $\to$ kadu   \\
6  & k  & $\emptyset$ & Qual.~(s)  & puku $\to$ puu    \\
7  & mp & mm  & Qual.~(c) & kampa $\to$ kamma  \\
8  & lt & ll  & Qual.~(c) & kulta $\to$ kulla  \\
9  & nt & nn  & Qual.~(c) & ranta $\to$ ranna  \\
10 & rt & rr  & Qual.~(c) & parta $\to$ parra  \\
11 & nk & ng  & Qual.~(c) & kenk\"a $\to$ keng\"a \\
\bottomrule
\end{tabular}
\caption{The 11 Finnish consonant gradation patterns (KOTUS types).
  Quant.\ = quantitative (geminate reduction), Qual.\ = qualitative
  (s = single, c = cluster).
  CoKleisli arrows examine only the immediate left neighbor,
  consistent with input strictly local (ISL) functions
  \citep{chandlee2014strictly}.}
\label{tab:gradation-patterns}
\end{table}

\paragraph{Vowel harmony} (\emph{vokaalisointu}) divides Finnish vowels
into back (\emph{a, o, u}), front (\emph{\"a, \"o, y}), and neutral
(\emph{e, i}). Suffixes with archiphonemes A, O, U are realized as
back or front variants depending on the stem. Vowel harmony exhibits
TSL-2 locality \citep{heinz2018computational}: the harmony arrow scans
leftward through transparent neutral vowels, depending only on the
nearest non-neutral vowel on the vowel tier (a window of size~2 on that
tier). The Zipper's full left context naturally supports this tier
projection: the harmony arrow scans leftward past neutral vowels until a
non-neutral vowel is found, implementing this tier-based locality within
the standard coKleisli mechanism.

\paragraph{Possessive suffix vowel copying.}
The archiphoneme \emph{V} is realized as a copy of the immediately
preceding vowel: the third-person possessive suffix \emph{-Vn} becomes
\emph{-an} after \emph{a}, \emph{-en} after \emph{e}, and analogously for the remaining vowels.

\subsection{Comonads in Computer Science}
\label{sec:comonads}

A \emph{comonad} is dual to a monad: where monads abstract over
computations that produce effects, comonads abstract over computations
that consume context \citep{orchard2012notation}. We use the
$\mathsf{extract}$/$\mathsf{extend}$ formulation:

\begin{definition}[Comonad]
\label{def:comonad}
A comonad $(W, \mathsf{extract}, \mathsf{extend})$ consists of an
endofunctor $W$ with $\mathsf{extract} : W\,a \to a$ and
$\mathsf{extend} : (W\,a \to b) \to W\,a \to W\,b$, satisfying:
{\small
\begin{align}
  \mathsf{extend}\;\mathsf{extract} &= \mathrm{id}
    \tag{L1$'$} \label{eq:law1p} \\
  \mathsf{extract} \circ \mathsf{extend}\;f &= f
    \tag{L2$'$} \label{eq:law2p} \\
  \mathsf{extend}\;g \circ \mathsf{extend}\;f
    &= \mathsf{extend}\;(g \circ \mathsf{extend}\;f)
    \tag{L3$'$} \label{eq:law3p}
\end{align}}
\end{definition}

\begin{definition}[CoKleisli arrow]
\label{def:cokleisli}
A function $f : W\,a \to b$ is a \emph{coKleisli arrow}.
CoKleisli arrows compose via
$(f \mathbin{>\!\!=\!\!>} g)(w) = g(\mathsf{extend}\;f\;w)$,
forming the \emph{coKleisli category} of $W$ (associative,
with $\mathsf{extract}$ as identity).
\end{definition}

\paragraph{The ListZipper comonad.} The ListZipper \citep{huet1997zipper}
represents a non-empty sequence with a distinguished \emph{focus}
and left/right context (Figure~\ref{fig:zipper-structure}).
$\mathsf{extract}$ returns the focused element;
$\mathsf{extend}\;f$ applies $f$ at every position, each seeing
the full context from that position's perspective.
\citet{uustalu2005essence} showed that this captures the essence of
applying a local update rule globally---precisely a cellular automaton.
Our insight is that Finnish morphophonological rules instantiate the
same pattern.

%

\begin{figure*}[t]
\centering
\begin{tikzpicture}[
    cell/.style={
      draw, minimum width=1.1cm, minimum height=0.9cm,
      font=\ttfamily\small, outer sep=0pt,
    },
    focus/.style={cell, fill=black!12, line width=0.8pt},
    ctxlabel/.style={font=\footnotesize\itshape, text=black!70},
    arrow/.style={-{Stealth[length=5pt]}, thick},
  ]

  \node[cell]  (c0) at (0,0)    {k};
  \node[cell]  (c1) at (1.1,0)  {a};
  \node[cell]  (c2) at (2.2,0)  {a};
  \node[focus] (c3) at (3.3,0)  {p};
  \node[cell]  (c4) at (4.4,0)  {p};
  \node[cell]  (c5) at (5.5,0)  {i};

  \draw[arrow] (3.3, -0.8) -- (c3.south);
  \node[font=\footnotesize\bfseries] at (3.3, -1.05) {focus};

  \draw[decorate, decoration={brace, amplitude=5pt, mirror}]
    ([yshift=-1.5cm]c0.south west) -- ([yshift=-1.5cm]c2.south east)
    node[midway, below=6pt, ctxlabel] {left context (reversed)};

  \draw[decorate, decoration={brace, amplitude=5pt, mirror}]
    ([yshift=-1.5cm]c4.south west) -- ([yshift=-1.5cm]c5.south east)
    node[midway, below=6pt, ctxlabel] {right context};

  \node[font=\small, anchor=west] at (-0.8, 1.3)
    {$z_3 = \mathsf{Zipper}\bigl(
      [\mathtt{k},\mathtt{a},\mathtt{a}],\;
      \mathtt{p},\;
      [\mathtt{p},\mathtt{i}]
    \bigr)$};

  \node[font=\footnotesize, anchor=west] (ext) at (6.6, 0.35)
    {$\mathsf{extract}(z_3) = \mathtt{p}$};
  \node[font=\footnotesize, anchor=west] (dup) at (6.6, -0.35)
    {$\mathsf{extend}(f, z)$: apply $f$ at every position};

\end{tikzpicture}
\caption{The ListZipper data structure for the word \emph{kaappi} focused
  on position~3. The left context is stored in reversed order
  (nearest neighbor last) for $O(1)$ access.
  $\mathsf{extract}$ returns the focused element;
  $\mathsf{extend}(f, z)$ applies a coKleisli arrow $f$ at every
  position, each seeing the full context from that position's
  perspective.}
\label{fig:zipper-structure}
\end{figure*}

\begin{figure*}[t]
\centering
\resizebox{\textwidth}{!}{%
\begin{tikzpicture}[
    box/.style={
      draw, rounded corners=3pt, minimum width=3.2cm, minimum height=1.2cm,
      font=\small, align=center, line width=0.6pt,
    },
    io/.style={font=\small\ttfamily},
    arrow/.style={-{Stealth[length=5pt]}, thick},
    extlbl/.style={font=\scriptsize, text=black!60, above=1pt},
  ]

  \node[io] (input) at (-1.6, 0) {input};

  \node[box] (cg) at (2.2, 0) {Consonant\\Gradation};
  \node[box] (vh) at (6.4, 0) {Vowel\\Harmony};
  \node[box] (ps) at (10.6, 0) {Possessive\\Suffix};

  \node[io] (output) at (14.2, 0) {output};

  \draw[arrow] (input.east) -- (cg.west);
  \draw[arrow] (cg.east) -- (vh.west)
    node[midway, above=3pt, font=\scriptsize] {$\mathsf{extend}$};
  \draw[arrow] (vh.east) -- (ps.west)
    node[midway, above=3pt, font=\scriptsize] {$\mathsf{extend}$};
  \draw[arrow] (ps.east) -- (output.west);

  \node[font=\scriptsize\ttfamily, text=black!60] at (2.2, -0.9)
    {apply\_gradation};
  \node[font=\scriptsize\ttfamily, text=black!60] at (6.4, -0.9)
    {apply\_harmony};
  \node[font=\scriptsize\ttfamily, text=black!60] at (10.6, -0.9)
    {apply\_possessive};

  \node[font=\scriptsize, text=black!50] at (2.2, -1.3)
    {$W\,a \to b$};
  \node[font=\scriptsize, text=black!50] at (6.4, -1.3)
    {$W\,a \to b$};
  \node[font=\scriptsize, text=black!50] at (10.6, -1.3)
    {$W\,a \to b$};

  \node[font=\small, anchor=north] at (6.4, -1.9)
    {$\mathsf{pipe} =
      \mathsf{grad}
        \mathbin{>\!\!=\!\!>}
      \mathsf{harmony}
        \mathbin{>\!\!=\!\!>}
      \mathsf{poss}$};

  \node[font=\scriptsize\itshape, text=black!55, anchor=north] at (6.4, -2.5)
    {Associative: $(f \mathbin{>\!\!=\!\!>} g) \mathbin{>\!\!=\!\!>} h
     = f \mathbin{>\!\!=\!\!>} (g \mathbin{>\!\!=\!\!>} h)$};

\end{tikzpicture}%
}
\caption{Morphophonological pipeline as coKleisli composition.
  Each box is a coKleisli arrow ($W\,a \to b$) lifted globally
  by $\mathsf{extend}$. CoKleisli composition is associative
  with $\mathsf{extract}$ as identity.}
\label{fig:cokleisli-pipeline}
\end{figure*}
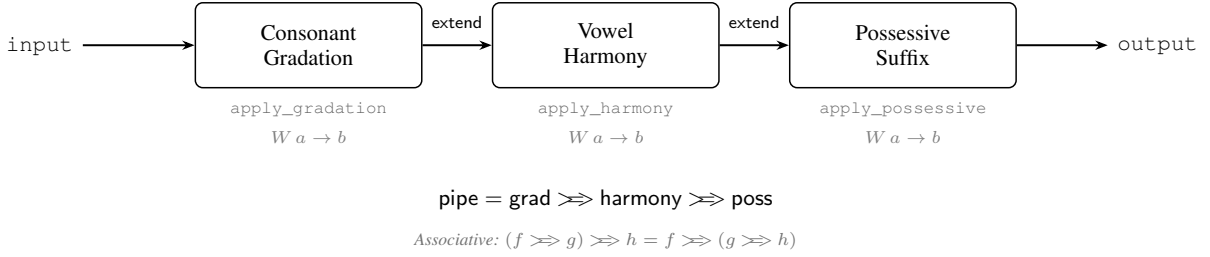

\begin{figure*}[t]
\centering
\begin{tikzpicture}[
    cell/.style={
      draw, minimum width=0.8cm, minimum height=0.7cm,
      font=\ttfamily\footnotesize, outer sep=0pt,
    },
    focus/.style={cell, fill=black!12, line width=0.8pt},
    dead/.style={cell, fill=black!5, text=black!40},
    arrow/.style={-{Stealth[length=4pt]}, thick},
    note/.style={font=\scriptsize, text=black!65},
    steplbl/.style={font=\footnotesize\bfseries, anchor=east},
  ]

  \def\rowsep{-1.5}

  \node[steplbl] at (-0.6, 0) {Input};
  \node[cell]  (i0) at (0.4, 0)  {k};
  \node[cell]  (i1) at (1.2, 0)  {a};
  \node[cell]  (i2) at (2.0, 0)  {a};
  \node[cell]  (i3) at (2.8, 0)  {p};
  \node[cell]  (i4) at (3.6, 0)  {p};
  \node[cell]  (i5) at (4.4, 0)  {i};
  \node[note, anchor=west] at (5.2, 0) {\emph{kaappi} (strong grade)};

  \node[steplbl] at (-0.6, \rowsep) {$\mathsf{extend}$};

  \node[cell]  (e0) at (0.4, \rowsep)  {k};
  \node[cell]  (e1) at (1.2, \rowsep)  {a};
  \node[cell]  (e2) at (2.0, \rowsep)  {a};
  \node[focus] (e3) at (2.8, \rowsep)  {p};
  \node[dead]  (e4) at (3.6, \rowsep)  {p};
  \node[font=\scriptsize\bfseries, text=red!70!black, anchor=south] at (3.6, \rowsep+0.35) {$\times$};
  \node[cell]  (e5) at (4.4, \rowsep)  {i};

  \draw[densely dashed, black!40] (e3.south) -- ++(0, -0.3)
    -- ++(1.6, 0) node[right, note] {pp detected: \texttt{p} suppressed};

  \draw[arrow, black!70] (i4.south) -- (e4.north);
  \node[note, anchor=west] at (5.2, \rowsep+0.15)
    {left = \texttt{p} $\Rightarrow$ geminate pp};
  \node[note, anchor=west] at (5.2, \rowsep-0.15)
    {$\mathit{DeletionSet} \mathbin{\cup}= \{4\}$};

  \node[steplbl] at (-0.6, 2*\rowsep) {Materialize};
  \node[cell]  (f0) at (0.4, 2*\rowsep)  {k};
  \node[cell]  (f1) at (1.2, 2*\rowsep)  {a};
  \node[cell]  (f2) at (2.0, 2*\rowsep)  {a};
  \node[cell]  (f3) at (2.8, 2*\rowsep)  {p};
  \node[cell]  (f4) at (3.6, 2*\rowsep)  {i};
  \node[note, anchor=west] at (4.4, 2*\rowsep) {\emph{kaapi} (weak grade)};

  \node[note, anchor=north west] at (-0.2, 2*\rowsep - 0.55)
    {+ genitive suffix \emph{-n} $\Rightarrow$ \emph{kaapin} (``of the cupboard'')};

\end{tikzpicture}
\caption{Step-by-step trace of consonant gradation on \emph{kaappi}
  using the Writer comonad (positions are 0-indexed).
  The $\mathsf{extend}$ pass applies the gradation arrow at each position.
  At position~3, the first \texttt{p} is recognized as position~0 of
  the geminate \texttt{pp} pattern and \textbf{suppressed} (not
  transformed to \texttt{v}).
  At position~4, the second \texttt{p} matches position~1 of the
  geminate pattern; the arrow returns
  $(\{4\}, \texttt{p})$, marking it for deletion via the
  $\mathit{DeletionSet}$ while preserving the character in the zipper.
  $\mathsf{Materialize}$ applies accumulated deletions once, yielding
  \emph{kaapi} (weak grade stem).}
\label{fig:kaappi-trace}
\end{figure*}
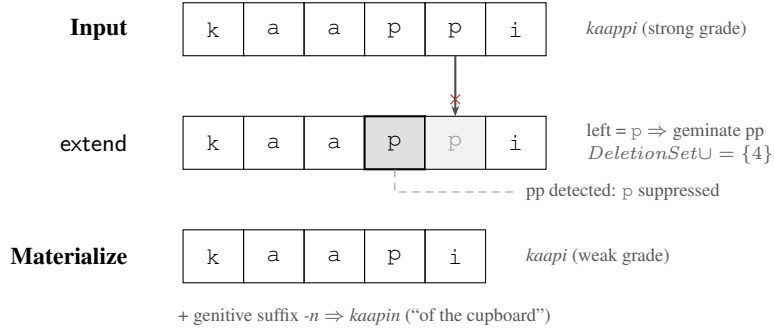

\section{Formalization}
\label{sec:formalization}

\subsection{Consonant Gradation as CoKleisli Arrows}
\label{sec:gradation}

We encode each consonant gradation pattern as a two-character window
$(c_0, c_1)$, where $c_0$ is the left context and $c_1$ the focused
character. The coKleisli arrow has the signature:
\begin{equation}
  \mathsf{apply\_gradation} :
    \mathsf{Zipper}(\mathit{char}) \times \mathit{Grade} \to \mathit{char}
  \label{eq:apply-gradation}
\end{equation}
where $\mathit{Grade} \in \{\mathsf{Strong}, \mathsf{Weak}\}$.
Partially applying the grade yields a coKleisli arrow
$\mathsf{Zipper}(\mathit{char}) \to \mathit{char}$ suitable for
$\mathsf{extend}$.

\paragraph{Design principle: only position~1 is transformed.}
The coKleisli arrow modifies only the focused character (position~1 of
the pattern window) and treats the left neighbor (position~0) as
read-only context. This is essential because $\mathsf{extend}$ applies
the function at \emph{every} position: if both characters were eligible,
geminate $pp \to p$ would produce double-counting. By transforming only
position~1, each position produces exactly one output, consistent with
comonadic $\mathsf{extend}$ semantics.
This mirrors cellular automata, where each cell reads neighbors but
writes only its own state.

\paragraph{The deletion marker.}
Pattern~6 (intervocalic \emph{k} deletion) presents a type-theoretic
challenge: the coKleisli arrow has type
$\mathsf{Zipper}(\mathit{char}) \to \mathit{char}$, requiring each
position to produce exactly one character. A na\"ive approach inserts a
null character \verb|'\0'| as a deletion marker, but this breaks strict
coKleisli associativity because filtering between steps requires
materializing and re-creating zippers. We resolve this algebraically via
the Writer comonad (\S\ref{sec:writer-comonad}).

\paragraph{Pattern priority and suppression.}
Priority is determined by specificity: geminates match first, then
clusters, then single consonants. A suppression predicate
$\mathsf{is\_pos0}(c, r)$ returns true when the focused character $c$
and its right neighbor $r$ form position~0 of a higher-priority
pattern (e.g., $\mathsf{is\_pos0}(\texttt{p}, \texttt{p}) = \text{true}$
for the geminate \texttt{pp}), preventing single-consonant rules from
firing erroneously. Formally:
{\small
\begin{equation}
  \mathsf{grad}(z, W) \!=\!
  \begin{cases}
    c & \mathsf{is\_pos0}(c, r) \\
    \mathsf{tgt}(p, W)[1]
      & \mathsf{find}(l, c, r, W) \!=\! \mathsf{Some}(p) \\
    c & \text{otherwise}
  \end{cases}
  \label{eq:gradation-formal}
\end{equation}}

\subsection{Worked Example: Weakening ``kaappi''}
\label{sec:worked-example}

We trace
$\mathsf{extend}(\mathsf{apply\_gradation}(\cdot,\mathsf{Weak}))$
on \emph{kaappi} in Figure~\ref{fig:kaappi-trace}. The key interactions are: (1)~at position~3, the
first \texttt{p} is \emph{suppressed} because
$\mathsf{is\_pos0}(\texttt{p}, \texttt{p})$ recognizes it as position~0
of the geminate pattern---without suppression, the single-consonant
rule would fire, yielding incorrect \emph{*kaavi}; (2)~at position~4,
the geminate pattern matches and the Writer comonad arrow returns
$(\{4\}, \texttt{p})$, marking the position for deletion;
(3)~$\mathsf{materialize}$ applies the accumulated
$\mathit{DeletionSet} = \{4\}$, yielding \emph{kaapi} (weak grade).

\subsection{Vowel Harmony as a CoKleisli Arrow}
\label{sec:vowel-harmony}

Vowel harmony is a coKleisli arrow
$\mathsf{harmony} : \mathsf{Zipper}(\mathit{char}) \to
\mathit{char}$ resolving archiphonemes by harmony class:
\begin{equation}
  \mathsf{detect}(z) \!=\!
  \begin{cases}
    \mathsf{Back}  & \text{back vowel left} \\
    \mathsf{Front} & \text{front vowel left} \\
    \mathsf{Front} & \text{neutral only}
  \end{cases}
  \label{eq:detect-harmony}
\end{equation}
\begin{equation}
  \mathsf{harmony}(z) \!=\!
  \begin{cases}
    a / \text{\"a} & \mathsf{extract}(z) \!=\! \texttt{A} \\
    o / \text{\"o} & \mathsf{extract}(z) \!=\! \texttt{O} \\
    u / y          & \mathsf{extract}(z) \!=\! \texttt{U} \\
    c              & \text{otherwise}
  \end{cases}
  \label{eq:vowel-harmony}
\end{equation}

\begin{example}
\emph{talo} + \emph{-ssA} $\to$ \texttt{t-a-l-o-s-s-a} (back vowel
\emph{a} found left). Conversely, \emph{kyn\"a} + \emph{-ssA} $\to$
\texttt{k-y-n-\"a-s-s-\"a} (front vowels \emph{y, \"a}).
\end{example}

\subsection{Pipeline Composition and Bidirectional Morphology}
\label{sec:pipeline}

The morphophonological pipeline chains three coKleisli arrows
(Figure~\ref{fig:cokleisli-pipeline}):
\begin{equation}
  \mathsf{pipe} = \mathsf{grad}
    \mathbin{>\!\!=\!\!>} \mathsf{harmony}
    \mathbin{>\!\!=\!\!>} \mathsf{poss}
  \label{eq:pipeline}
\end{equation}
A na\"ive implementation inserts a \verb|'\0'|-filtering step between
gradation and harmony, which breaks strict associativity. The Writer
comonad (\S\ref{sec:writer-comonad}) resolves this: each arrow returns
$(\mathit{DeletionSet}, \mathit{char})$, accumulated deletions are
merged via monoid union during $\mathsf{extend}$, and a single
$\mathsf{materialize}$ step applies all deletions at pipeline end.
This restores strict coKleisli associativity across the full pipeline,
including deletion.

\begin{example}[Pipeline: ``kampAstAVn'' (weak grade)]
(1)~Gradation: $\texttt{mp} \to \texttt{mm}$ $\Rightarrow$
\texttt{kammAstAVn}.
(2)~Harmony: $\texttt{A} \to \texttt{a}$ $\Rightarrow$
\texttt{kammastaVn}.
(3)~Possessive: $\texttt{V} \to \texttt{a}$ $\Rightarrow$
\texttt{kammastaan} (`from his/her comb').
\end{example}

\paragraph{MorphGenerator: bidirectional morphology.}
The coKleisli pipeline enables bidirectional morphology: analysis (via
Voikko FST, hereafter VFST) and generation (via coKleisli arrows) share the same rules.
The MorphGenerator produces inflected forms from baseform + features
by applying coKleisli arrows in the generative direction (currently
11 noun cases). Since each rule is an explicit function rather than
a compiled FST state, the same logic drives both analysis and
generation with minimal additional code.

\paragraph{Rule ordering.}
Gradation precedes harmony because gradation may delete characters,
changing the vowel context. The coKleisli framework does not eliminate
order-dependence---it formalizes it: each rule remains individually
addressable and testable, and the composition order is explicit rather
than implicit in merged automaton states.

\subsection{CG Rules as Comonadic Extension}
\label{sec:cg-rules}

The comonadic framework extends to sentence-level disambiguation: only
the Zipper's element type changes (from character to ReadingSet),
while the coKleisli composition structure is identical (the
deletion-handling Writer comonad needed for character-level rules is
formalized in \S\ref{sec:writer-comonad}). In CG-lite, each rule is a
coKleisli arrow over \emph{reading sets}:
\begin{equation}
  \mathit{rule} : \mathsf{Zipper}(\mathit{ReadingSet}) \to \mathit{ReadingSet}
  \label{eq:cg-rule-type}
\end{equation}
Twenty-four rule types are implemented, covering removal, selection,
and structural patterns parameterized by part-of-speech (POS) tags, baseforms, and
feature constraints. The current Finnish configuration contains 62
active rules (85 total; 23 disabled for accuracy after tuning).
All rules enforce a safety invariant: a rule never removes the last
reading at any position, ensuring well-definedness of subsequent
$\mathsf{extend}$ passes. The same coKleisli abstraction that
governs character-level morphophonology applies at this sentence
level---each CG rule is a coKleisli arrow over reading sets, composed
via $\mathsf{extend}$ (details in Appendix~\ref{app:eval-cg-examples}).

Multiple CG rules are applied by iterating $\mathsf{extend}$:
\begin{multline}
  \mathsf{cg}(s, [r_1, \ldots, r_k]) = \\
    \mathsf{vec}(\mathsf{ext}(r_k, \ldots
      \mathsf{ext}(r_1,
        \mathsf{Zip}(s)) \ldots))
  \label{eq:cg-sequential}
\end{multline}
Note that sequential extension is equivalent to coKleisli composition:
$\mathsf{extend}(r_2) \circ \mathsf{extend}(r_1) =
\mathsf{extend}(r_1 \mathbin{>\!\!=\!\!>} r_2)$, so the CG pipeline
uses the same composition mechanism as the morphophonological pipeline
(\eqref{eq:pipeline}).

\paragraph{Unified comonadic abstraction.}
The same algebraic structure operates at three linguistic levels
(Table~\ref{tab:unified-levels}):

\begin{table}[t]
\centering
\small
\begin{tabular}{@{}llll@{}}
\toprule
Level & Zipper & Arrow & Out \\
\midrule
Char & $\mathsf{Zip}\langle\mathit{c}\rangle$
  & Gradation & $\mathit{c}$ \\
Morph & $\mathsf{Zip}\langle\mathit{c}\rangle$
  & Harmony & $\mathit{c}$ \\
Sent & $\mathsf{Zip}\langle\mathit{RS}\rangle$
  & CG rule & $\mathit{RS}$ \\
\bottomrule
\end{tabular}
\caption{Unified comonadic abstraction across three
  linguistic levels ($c$ = char, $RS$ = ReadingSet).
  Zipper provides context, the coKleisli arrow computes
  a local transformation, $\mathsf{extend}$ lifts it globally.
  These three levels are operationalized in MCE at: char/morph in the
  MorphGenerator (\S\ref{sec:pipeline}), sent in CG-lite disambiguation
  (\S\ref{sec:cg-rules}).}
\label{tab:unified-levels}
\end{table}

To our knowledge, this is the first application of the
comonad/cellular-automaton correspondence
\citep{capobianco2010categorical} simultaneously at multiple linguistic
levels in computational morphology.

\subsection{Writer Comonad for Deletion}
\label{sec:writer-comonad}

The deletion marker approach (\S\ref{sec:gradation}) would break strict
coKleisli associativity: filtering \verb|'\0'| characters between
$\mathsf{extend}$ passes requires materializing intermediate sequences
and reconstructing zippers, violating the comonad law
$\mathsf{extend}\;g \circ \mathsf{extend}\;f
  = \mathsf{extend}\;(g \circ \mathsf{extend}\;f)$.

The key idea: instead of immediately deleting characters, we
\emph{accumulate} deletion positions alongside the computation, applying
all deletions once at pipeline end. We formalize this with the
\emph{Writer comonad}\footnote{The name ``Writer comonad'' is by analogy
with the Writer monad---both accumulate a monoid value across a
computation---but the construction is not the categorical dual of
Writer. Standard categorical taxonomy does not provide a canonical name
for this structure; we retain ``Writer comonad'' for its descriptive
value with this caveat.}
$\mathit{DS} \times \mathsf{Zipper}$, where
$(\mathit{DS}, \cup, \emptyset)$ is the \emph{deletion monoid}---a set
of positions with union as the monoid operation and the empty set as
identity.

\begin{definition}[Writer Comonad]
\label{def:writer-comonad}
Let $(\mathit{DS}, \cup, \emptyset)$ be the deletion monoid.
The \emph{Writer comonad} is the product
$\mathit{DS} \times \mathsf{Zipper}$ with:
{\small
\begin{align}
  \mathsf{extract}_W(w, z) &= \mathsf{extract}(z) \label{eq:writer-extract} \\
  \mathsf{extend}_W(f)(w, z) &=
    \Bigl(w \cup \textstyle\bigcup_{i}\pi_1(f(w, z_i)),\notag\\
    &\qquad\;\mathsf{fmap}(\pi_2 \circ f)(w, z)\Bigr) \notag
\end{align}}
where $z_i$ ranges over all positions of $z$, $\pi_1$/$\pi_2$ are
the product projections onto $\mathit{DS}$ and the output character,
and $\mathsf{fmap}$ denotes the underlying Zipper's $\mathsf{extend}$
applied to the character component: concretely,
$\mathsf{fmap}(\pi_2 \circ f)(w, z) = \mathsf{Zipper.extend}(\lambda z_i.\;\pi_2(f(w, z_i)))(z)$.
\end{definition}

The construction factorizes: the Zipper satisfies standard comonad
laws on characters, while $\mathit{DS}$ satisfies the monoid laws
independently; $\mathsf{extend}_W$ preserves both.

Concretely, each coKleisli arrow returns $(\mathit{DS}, A)$ instead of
$A$: gradation returns $(\{\mathit{pos}\}, c)$ when deleting a consonant
and $(\emptyset, c')$ otherwise; harmony and possessive always return
$(\emptyset, c')$ since they are non-deleting.
The pipeline begins with $w = \emptyset$ (no prior deletions);
subsequent arrows in the coKleisli composition receive the accumulated
$\mathit{DS}$ from all preceding arrows via $\mathsf{extend}_W$,
propagating deletion information through the pipeline without
intermediate materialization.
A single $\mathsf{materialize}$ step at pipeline end applies all
accumulated deletions.

\begin{proposition}
\label{prop:writer-comonad-laws}
The Writer comonad construction $(\mathit{DS} \times \mathsf{Zipper},\;
\mathsf{extract}_W,\; \mathsf{extend}_W)$ satisfies the comonad laws
on the character component and the monoid laws on $\mathit{DS}$,
jointly ensuring compositional correctness of the pipeline.
\end{proposition}
\begin{proof}[Proof sketch]
\textbf{L1$'$} ($\mathsf{extend}_W\;\mathsf{extract}_W = \mathrm{id}$):
For the deletion component, $\pi_1(\mathsf{extract}_W(w, z_i)) = \emptyset$
at every position, so
$w \cup \bigcup_i \emptyset = w$ by monoid identity.
For the character component,
$\mathsf{Zipper.extend}(\lambda z_i.\;\pi_2(\mathsf{extract}_W(w, z_i)))(z)
= \mathsf{Zipper.extend}(\mathsf{extract})(z) = z$
by the underlying Zipper's L1$'$.
Thus $\mathsf{extend}_W(\mathsf{extract}_W)(w,z) = (w,z)$.
\textbf{L2$'$} ($\mathsf{extract}_W \circ \mathsf{extend}_W\;f = \pi_2 \circ f$):
$\mathsf{extract}_W(\mathsf{extend}_W(f)(w,z))
= \mathsf{extract}(\mathsf{Zipper.extend}(\lambda z_i.\;\pi_2(f(w,z_i)))(z))
= \pi_2(f(w,z))$
by the Zipper's L2$'$; the $\mathit{DS}$ component is accumulated by
$\mathsf{extend}_W$ and consumed by $\mathsf{materialize}$.
\textbf{L3$'$}: see Appendix~\ref{app:l3-proof} for the full argument;
the key step is that $\mathsf{extend}_W$ distributes over coKleisli
composition because the $\mathit{DS}$ components associate via
the monoid law $(\mathit{DS},\cup,\emptyset)$ and the character
components associate via the underlying Zipper's L3$'$.
All laws are verified by 44 unit tests covering the Zipper comonad
laws (character component) and the $\mathit{DS}$ monoid laws
independently.
\end{proof}

The equivalence between the Writer pipeline and the old hybrid pipeline
is verified on all 11 gradation patterns, confirming that the algebraic
treatment of deletion introduces no behavioral difference.

\section{Implementation}
\label{sec:implementation}

We implement the framework in the Morphological Computation Engine
(MCE), a Rust (2024 edition, stable 1.86+) library targeting
WebAssembly (${\sim}$380\,KB).
Since Rust lacks higher-kinded types, the comonad operations are
implemented directly on the concrete \texttt{Zipper<T>} type rather
than via a generic \texttt{Comonad} trait. Each morphophonological rule
is a plain function \texttt{fn(\&Zipper<char>) -> char}; CG rules use
trait-based dispatch with the same coKleisli signature. The current
$\mathsf{extend}$ allocates $O(n^2)$ due to cloning at each position,
which is negligible for Finnish word lengths ($n \leq 15$). The
implementation is verified through 317 tests (comonadic pipeline;
1{,}619 total across the system) covering Zipper comonad laws,
$\mathit{DeletionSet}$ monoid laws, all 11 gradation
patterns in both directions, roundtrip properties, Writer pipeline
equivalence, and CG rule correctness. Full implementation details,
Rust code listings, and test methodology are provided in
Appendix~\ref{app:implementation}.

\section{Evaluation}
\label{sec:evaluation}

\subsection{Experimental Setup}
\label{sec:eval-setup}

We evaluate on UD Finnish-TDT v2.14 \citep{haverinen2014tdt},
distributed under CC-BY-SA 4.0, dev split
(1{,}364 sentences, ${\sim}$25{,}000 tokens) with gold tokenization.
Since our CG rules are hand-written and the suffix tagger was trained
exclusively on the training split, the development set provides an
unbiased evaluation.
Accuracy is micro-averaged token-level accuracy; out-of-vocabulary (OOV) tokens (0.36\%)
default to NOUN. The morphological analyzer uses the Voikko Finnish
morphological dictionary (voikko-fi v2.6, 39{,}607 lexical entries;
\texttt{mor.vfst}, 4\,MB). All experiments run on a single
CPU core (Apple M2 Max); no GPU is required.

\subsection{Morphophonological Correctness}
\label{sec:eval-correctness}

All 11 consonant gradation patterns are correctly formalized as
coKleisli arrows, verified by the Writer comonad pipeline. Of these, 7
exhibit full bidirectional roundtrip correctness
($\mathsf{strong}(\mathsf{weak}(w)) = w$); the 4 non-roundtrip cases
all involve deletion, confirming deletion as the sole source of
irreversibility. The Writer comonad laws hold as follows.
L1$'$ (left identity) and L2$'$ (right identity) hold for arbitrary
coKleisli arrows. L3$'$ (associativity) holds on the value (Zipper)
component; the log component is treated as an accumulating output
rather than a structural part of the comonad, and is materialized once
at the end of each pipeline stage. In practice we invoke
$\mathsf{extend}$ once per stage, so the question of log-component
associativity under repeated $\mathsf{extend}$ does not arise in our
pipeline; the Zipper alone satisfies all three comonad laws.
Forty-four dedicated tests verify these properties. Full
pattern-level results are in Appendix~\ref{app:eval-qualitative}.

\subsection{System Performance}
\label{sec:eval-results}

\paragraph{Throughput.}
The complete MCE pipeline processes \textbf{84{,}973 tokens/sec}
on Apple M-series (${\sim}$12\,$\mu$s/token, ${\sim}$0.21\,ms for a
15-word sentence), demonstrating that coKleisli composition introduces
no significant overhead compared to monolithic implementations.

\paragraph{Per-rule latency.}
Table~\ref{tab:per-rule-latency} reports per-rule microbenchmarks
(10K iterations). The full morphophonological pipeline averages
2.66\,$\mu$s per word, confirming that coKleisli composition introduces
negligible overhead beyond the sum of its components.

\begin{table}[t]
\centering
\small
\begin{tabular}{@{}lrr@{}}
\toprule
Component & Mean & Std \\
 & ($\mu$s) & ($\mu$s) \\
\midrule
\multicolumn{3}{@{}l}{\emph{CoKleisli (char-level)}} \\
\quad Gradation (avg/11) & 0.61 & 0.24 \\
\quad Harmony (avg/4) & 0.85 & 0.22 \\
\quad Possessive (avg/3) & 0.70 & 0.24 \\
\quad Full pipeline (avg/4) & 2.66 & 0.58 \\
\midrule
\multicolumn{3}{@{}l}{\emph{CG-lite (sentence-level)}} \\
\quad Single rule (avg/20) & 7.66 & 1.34 \\
\quad Full CG (62 rules) & 376.04 & 32.35 \\
\bottomrule
\end{tabular}
\caption{Per-rule latency (10K iterations, Apple M-series).
  CoKleisli arrows: sub-$\mu$s; full pipeline: under 3\,$\mu$s.}
\label{tab:per-rule-latency}
\end{table}

\paragraph{UPOS accuracy.}

\begin{table}[t]
\centering
\small
\begin{tabular}{@{}lrrl@{}}
\toprule
System & UPOS & Lemma & Approach \\
\midrule
TNPP\textsuperscript{$\dagger$} & 96.91 & 95.54
  & Neural \\
\textbf{MCE-full (ours)} & \textbf{94.66} & \textbf{93.09}
  & \textbf{FST+CG+stat} \\
Omorfi 1-best & 83.88 & 82.63
  & FST + freq. \\
MCE-rules (ours) & 83.92 & 93.09
  & FST+CG \\
\bottomrule
\end{tabular}
\caption{UPOS accuracy (\%) on UD Finnish-TDT dev set with
  gold tokenization. Coverage: 99.35\%.
  MCE-rules uses only comonadic CG disambiguation (62 active rules);
  MCE-full adds a lightweight suffix tagger (\S\ref{sec:eval-ablation}).
  $\dagger$TNPP = Turku Neural Parser Pipeline.
  TNPP and Omorfi figures from \citet{pirinen2019neural}, evaluated
  on an earlier UD Finnish-TDT version; POS annotation guidelines have
  remained stable across versions, so direct comparison is approximate
  but meaningful.}
\label{tab:upos-comparison}
\end{table}

With rule-only disambiguation (62 active CG rules), MCE achieves 83.92\%
UPOS and 93.09\% lemma accuracy (Table~\ref{tab:upos-comparison}) with
99.35\% coverage. Adding a suffix tagger raises UPOS to 94.66\%,
narrowing the gap to TNPP from 12.99\,pp to 2.25\,pp.

\paragraph{Comparison with FST-based systems.}

\begin{table}[t]
\centering
\small
\begin{tabular}{p{1.8cm}p{2.1cm}p{2.4cm}}
\toprule
Property & Comonadic & FST cascade \\
\midrule
Composition & $O(k)$ functions
  & $O(|S_1| \!\times\! |S_2|)$ \\
Modularity & Independent fns
  & Merged FST \\
Debugging & Per-rule tests
  & Composite trace \\
Deletion & Writer comonad
  & $\varepsilon$-transitions \\
Rule count & 20 arrows$^{\ddagger}$
  & 8{,}157 LEXICONs \\
\bottomrule
\end{tabular}
\caption{Comonadic vs.\ FST-based rule encoding.
  $\ddagger$20 total arrows: 13 morphophonological (11 gradation +
  harmony + possessive) + 7 system-level (tokenization, OOV fallback,
  case mapping, compound boundary, AUX resolution, coverage, Viterbi
  integration); at the morphophonological level, 13 vs.\ 874 (67:1).}
\label{tab:fst-comparison}
\end{table}

Table~\ref{tab:fst-comparison} summarizes the key differences.
Voikko's 8{,}157 LEXICON entries encode the full lexicon (stems,
paradigms, compounds, \emph{and} morphophonological rules), so the
gross 408:1 ratio overstates the difference. At the
morphophonological-rule-only level, Omorfi uses 874 continuation
classes (810 harmony $\times$ suffix-class variants, 62 harmony $\times$
gradation, 2 gradation only)\footnote{Counted from Omorfi's
  \texttt{continuations.tsv} (commit \texttt{6ce67ac}).}
vs.\ MCE's 13 arrows (67:1 at the rule-representation level).
This comparison concerns rule-representation units: the MCE analysis
pipeline itself uses a lexicon-based analyzer (Voikko-derived VFST)
compatible with Omorfi-style approaches.
The insight is \emph{orthogonal composition}: FSTs require
$O(P \times H \times C)$ entries per combination, while coKleisli
composition is $O(P + H + C)$. The FST approach retains advantages within the same module: native
$\varepsilon$-transitions, minimization, and direct encoding of
lexical exceptions. Our framework operates in the SPE-era modular
tradition \citep{chomsky1968spe}, scoping itself to the rule layer and
delegating lexical idiosyncrasies to the lexicon---in MCE, the
Voikko-derived analyzer.

\paragraph{Developer ergonomics.}
Adding a new morphophonological rule requires defining a single
coKleisli arrow (${\sim}$5 lines of Rust) that automatically composes
with all existing rules. In contrast, adding a rule to an FST system
requires modifying potentially hundreds of \texttt{LEXICON} entries to
account for all interaction combinations---a direct consequence of the
multiplicative vs.\ additive composition difference.

\subsection{Ablation and Error Analysis}
\label{sec:eval-ablation}

\begin{table}[t]
\centering
\small
\begin{tabular}{@{}lrl@{}}
\toprule
Configuration & UPOS & $\Delta$ \\
\midrule
FST only (no disambig.) & 38.96 & --- \\
+ Fallback + coverage & ${\sim}$70 & +31 \\
+ AUX/participle map & ${\sim}$80 & +10 \\
+ Viterbi bigram & ${\sim}$81.20 & +1.2 \\
+ Priors + CG-lite & 83.92 & +2.72 \\
\midrule
+ Suffix tagger & 94.66 & +10.74 \\
\bottomrule
\end{tabular}
\caption{Ablation study (\% UPOS, $\Delta$ in pp).
  Values prefixed with ${\sim}$ are approximate, measured during
  incremental development before the final pipeline was frozen.
  The suffix tagger is a separate statistical component beyond the
  comonadic pipeline.}
\label{tab:ablation}
\end{table}

Table~\ref{tab:ablation} shows each component's contribution. The
CG-lite rules target top confusion pairs (NOUN/VERB, ADJ/NOUN,
ADV/NOUN, PRON/NOUN), each providing deterministic, explainable
disambiguation. Structural fixes (+31\,pp coverage, +10\,pp mapping)
contribute more than corpus bigrams (+1.2\,pp).

An optional suffix tagger (logistic regression, 157K features, 5\,MB)
adds 10.74\,pp by resolving ambiguities requiring distributional
evidence---primarily NOUN/VERB pairs (\emph{tuuli} `wind' vs.\
`blew'), contextual ADJ/NOUN distinctions, and verbal forms following
negation verbs where CG rules lack sufficient context.
It operates \emph{after} the comonadic pipeline using
CG-filtered candidates; it is a separate statistical module.
Detailed CG examples are in Appendix~\ref{app:eval-cg-examples}.

\section{Discussion}
\label{sec:discussion}
\label{sec:extensions}

The Cofree comonad could provide incremental re-analysis for interactive
spell-checking, and the comonad/monad adjunction $W \dashv M$ suggests
a connection between comonadic (context-consuming) morphophonology and
monadic (effect-producing) generation.
We conjecture that every ISL-$k$ function over strings can be expressed
as a coKleisli arrow with window size at most $k{+}1$, establishing a
formal bridge between subregular phonology \citep{chandlee2018strict}
and comonadic composition. The intuition is direct: an ISL-$k$
function's output at each position depends on at most $k$ contiguous
input symbols; a coKleisli arrow over a Zipper of the same string has
access to the full context but need only examine $k{-}1$ neighbors,
matching the ISL sliding-window semantics; $\mathsf{extend}$ then lifts
this local computation to all positions simultaneously.
A formal proof requires careful treatment of boundary conditions
(string-initial and string-final positions where the Zipper context is
truncated) and is our primary direction for future work.
We further conjecture that the framework extends to other agglutinative
languages: Turkish 4-way vowel harmony, Hungarian ternary harmony,
and Estonian consonant gradation all exhibit local context-dependent
alternation expressible as coKleisli arrows. The framework should also
extend to other length-preserving phenomena: Korean consonant
assimilation, Japanese rendaku, North S\'ami consonant gradation.
Templatic morphology (Semitic root-and-pattern: Arabic K-T-B yields
\emph{kataba}, \emph{kitaab}, \emph{maktab}) and reduplication
(Indonesian \emph{orang-orang}, Tagalog \emph{bi-bili}) require segment
insertion or copying and are not directly captured by the present
framework; we develop an extension using an InsertionMap monoid in
follow-up work. Verification requires language-specific implementation,
which we leave to future work.

\section{Related Work}
\label{sec:related-work}

\paragraph{Finite-state morphology.}
\citet{koskenniemi1983two} established the finite-state paradigm,
formalized by \citet{kaplan1994regular} and systematized by
\citet{beesley2003finite}; \citet{hulden2009foma} and HFST \citep{linden2013hfst} provide efficient open-source FST tooling.
Our contribution is not a faster runtime but a \emph{composition algebra} that keeps rules
individually addressable after composition---something FST intersection
does not preserve.
Our framework operates in the SPE tradition \citep{chomsky1968spe}, orthogonal to OT's constraint ranking \citep{prince2004optimality}.

\paragraph{Functional Morphology.}
\citet{forsberg2004functional} treat morphological paradigms as Haskell
functions, achieving paradigm compression through functional abstraction.
Our framework provides a specific structural guarantee that general
functional composition lacks: $\mathsf{extend}$ uniformly lifts a local
rule to all positions with automatic context threading.
\citet{forsberg2004functional} focus on paradigm \emph{generation}; our
framework addresses \emph{application} of context-dependent rules,
though the MorphGenerator demonstrates that generation also fits
naturally.

\paragraph{Neural morphological transducers.}
Our work targets a different design point---interpretable, offline
morphological processing---and the distinct error profiles of neural
vs.\ rule-based systems suggest hybrid potential
\citep{pirinen2019neural}.
Recent neural transducers achieve strong results across diverse typologies
\citep[see][for a survey]{baxi2024computational};
SIGMORPHON shared tasks
\citep{kodner2022sigmorphon,goldman2023sigmorphon,batsuren2022sigmorphon}
continue to drive progress.
Combining comonadic rules with neural disambiguation is a natural direction.

\paragraph{Subregular hierarchy.}
Finnish consonant gradation is input strictly local (ISL); vowel
harmony exhibits TSL-2 locality
\citep{heinz2018computational,chandlee2014strictly}. Our gradation
arrows examine only the immediate left neighbor (ISL); the harmony arrow
scans through transparent vowels (matching TSL-2 locality on the vowel
tier). The subregular classification could inform arrow design for other
languages.

\paragraph{Category theory in NLP.}
\citet{defelice2022categorical} applies monoidal categories to compositional semantics; our contribution is orthogonal: \emph{comonadic} structure for context-dependent local transformation.

\section{Conclusion}
\label{sec:conclusion}

We have shown that the Writer comonad
($\mathsf{DeletionSet} \times \mathsf{Zipper}$) restores strict
coKleisli compositionality for length-changing morphophonological
rules, resolving algebraically a problem that FST cascades address only
through multiplicative state explosion---to our knowledge, the first
application of the comonad/cellular-automaton correspondence
\citep{capobianco2010categorical} to natural language morphophonology.

Several directions for future work emerge: extending the framework to
other agglutinative languages (\S\ref{sec:extensions}); enlarging the
coKleisli arrow vocabulary to cover epenthesis and metathesis for a more
complete algebraic classification; combining comonadic rules with neural
disambiguation to exploit complementary error profiles
\citep{pirinen2019neural}; and using the Cofree comonad for incremental
re-analysis in interactive spell-checking.

\paragraph{Code and data.}
The MCE implementation (Rust, Apache-2.0) and evaluation scripts will
be released at \url{https://github.com/yongsk0066/mce} upon publication,
with a version tag corresponding to the results reported here.
UD Finnish-TDT is distributed under CC-BY-SA~4.0.
The Voikko morphological dictionary (voikko-fi) is distributed under
GPL-2.0-or-later.

\section*{Limitations}
\label{sec:limitations}

The $\mathsf{Zipper}\langle\mathit{char}\rangle$ does not encode
morpheme boundaries; a $\mathsf{Zipper}\langle\mathit{Morpheme}\rangle$
would address this but requires a prior segmentation step.
Consonant deletion requires the Writer comonad
(\S\ref{sec:writer-comonad}), verified by 44 tests.
The 24 CG rule types do not cover the full CG specification's BARRIER
conditions and MAP/SUBSTITUTE operations
\citep{karlsson1990constraint,bick2000parsing}; extending the rule
vocabulary is straightforward within the coKleisli framework.
The rule-only 83.92\% is 12.99\,pp below neural TNPP (2.25\,pp with
the suffix tagger); the gap reflects deep syntactic disambiguation---not
a limitation of the comonadic framework itself.

\bibliography{references}


\appendix

\section{Implementation Details}
\label{app:implementation}

\subsection{Rust and the Absence of Higher-Kinded Types}
\label{app:rust-hkt}

The formalization in Section~\ref{sec:formalization} is presented in
terms of Haskell-style type classes, where a generic \texttt{Comonad}
class abstracts over any functor $W$. Rust lacks higher-kinded types
(HKTs), so a direct translation of \texttt{class Comonad w} is not
possible. We therefore adopt a pragmatic approach: the comonad operations
are implemented directly on the concrete \texttt{Zipper<T>} type,
without a generic \texttt{Comonad} trait.

\begin{lstlisting}[language=Rust,basicstyle=\scriptsize\ttfamily]
pub struct Zipper<T> {
    left:  Vec<T>,  // reversed
    focus: T,
    right: Vec<T>,
}

impl<T: Clone> Zipper<T> {
    pub fn extract(&self) -> &T {
        &self.focus
    }

    pub fn extend<U, F>(&self, f: F)
        -> Zipper<U>
    where F: Fn(&Zipper<T>) -> U {
        let focus = f(self);
        let left  = Lefts::new(self)
            .map(|z| f(&z)).collect();
        let right = Rights::new(self)
            .map(|z| f(&z)).collect();
        Zipper { left, focus, right }
    }
}
\end{lstlisting}

The \texttt{Lefts} and \texttt{Rights} are internal iterators that
produce successive left-shifted and right-shifted zippers by calling
\texttt{move\_left()} and \texttt{move\_right()} respectively. Each
shift clones the internal vectors, resulting in $O(n)$ work per shift
and $O(n^2)$ total allocation for a word of length $n$. For Finnish
words (typically 5--15 characters), this cost is negligible; our
benchmarks show that a 15-word sentence processes in approximately
0.21\,ms. A \texttt{SliceZipper} optimization that holds a reference to
the underlying slice and moves only a position index would reduce
intermediate allocation to $O(n)$.

This approach sacrifices generality---it is not possible to write code
that is polymorphic over arbitrary comonads---but gains clarity and
avoids the complex trait-bound machinery that Rust's Generic Associated
Types (GATs) would require for a full \texttt{Comonad} trait. Since our
framework requires only a single comonad instance (\texttt{Zipper}),
this trade-off is appropriate.

\subsection{CoKleisli Arrows in Rust}
\label{app:cokleisli-rust}

Each morphophonological rule is a plain Rust function with the signature
\texttt{fn(\&Zipper<char>) -> char}. For consonant gradation, the grade
parameter is supplied via partial application using a closure:

\begin{lstlisting}[language=Rust,basicstyle=\scriptsize\ttfamily]
pub fn apply_gradation(
    z: &Zipper<char>, grade: Grade
) -> char {
    let focus = *z.extract();
    let left  = z.peek_left(1).copied();
    let right = z.peek_right(1).copied();
    match find_pattern_at_pos1(
        left, focus, right, grade
    ) {
        Some(pat) => /* target char */,
        None      => focus,
    }
}
// Usage with extend:
let weak = zipper.extend(
    |z| apply_gradation(z, Grade::Weak)
);
\end{lstlisting}

The CG-lite rules use a trait-based dispatch instead:

\begin{lstlisting}[language=Rust,basicstyle=\scriptsize\ttfamily]
pub trait CgRule {
    fn apply(&self,
        z: &Zipper<ReadingSet>
    ) -> ReadingSet;
}
\end{lstlisting}

Each concrete rule type (24 types including \texttt{RemoveIfPreceded},
\texttt{SelectIfFollowed}, \texttt{RemoveByClass},
\texttt{SelectByBaseform}, and others) implements this
trait. The \texttt{apply} method has the coKleisli arrow signature
\texttt{\&Zipper<ReadingSet>}~$\to$~\texttt{ReadingSet}, and can be
passed directly to $\mathsf{extend}$:
\begin{lstlisting}[language=Rust,basicstyle=\scriptsize\ttfamily]
let result = zipper.extend(
    |z| rule.apply(z)
);
\end{lstlisting}

\subsection{The Gradation Pattern Table}
\label{app:pattern-table}

The 11 consonant gradation patterns are encoded in a static table of
\texttt{GradationPattern} structures, each specifying a two-character
window for both the strong and weak grades:

\begin{lstlisting}[language=Rust,basicstyle=\scriptsize\ttfamily]
const PATTERNS: &[GradationPattern] = &[
  // Geminates (highest priority)
  GP { s: ['p','p'], w: ['p','\0'] },
  GP { s: ['t','t'], w: ['t','\0'] },
  GP { s: ['k','k'], w: ['k','\0'] },
  // Clusters
  GP { s: ['m','p'], w: ['m','m'] },
  GP { s: ['n','t'], w: ['n','n'] },
  // ... (6 more patterns)
  // Single consonants (lowest priority)
  GP { s: ['\0','p'], w: ['\0','v'] },
  GP { s: ['\0','t'], w: ['\0','d'] },
  GP { s: ['\0','k'], w: ['\0','\0'] },
];
\end{lstlisting}

The ordering within the table encodes pattern priority: the first
matching pattern wins. The \verb|'\0'| character serves as an implementation
artifact in the internal representation: at position~0 of
single-consonant patterns, it indicates ``preceded by any vowel,''
distinguishing these patterns from cluster patterns that require a
specific left neighbor; as an output character, it signals deletion. When the
pattern table produces a \verb|'\0'| output, the Writer comonad arrow
(\S\ref{sec:writer-comonad}) intercepts it and records the position in
a $\mathit{DeletionSet}$ instead of emitting the null character,
restoring strict coKleisli compositionality. The \verb|'\0'| never
appears in final output---it exists solely as an internal sentinel
consumed by the Writer comonad's $\mathsf{materialize}$ step.

\subsection{Bidirectional Support}
\label{app:bidirectional}

All 11 patterns support both weakening (strong $\to$ weak) and
strengthening (weak $\to$ strong). The $\mathsf{find\_pattern\_at\_pos1}$
function selects which side of the pattern to match against based on the
$\mathit{Grade}$ parameter: for weakening, it matches the
$\mathsf{strong}$ side and produces the $\mathsf{weak}$ side; for
strengthening, vice versa.

\subsection{Test Methodology}
\label{app:tests}

The implementation is verified through 317 comonad-related tests
organized in six files:

\begin{table}[h]
\centering
\small
\begin{tabular}{@{}p{1.6cm}p{1.4cm}cp{2.6cm}@{}}
\toprule
Component & File & \# & Verification \\
\midrule
Zipper & \texttt{zipper} & 27
  & Laws, navigation \\
Morpho & \texttt{finnish} & 89
  & 11 patterns, roundtrip \\
Writer & \texttt{writer} & 44
  & Monoid/comonad laws, pipeline equiv. \\
CG-lite & \texttt{cg} & 138
  & 24 rule types, safety \\
Proptest & \texttt{proptest} & 12
  & Property-based law verification \\
Bench & \texttt{bench} & 7
  & Perf.\ regression \\
\bottomrule
\end{tabular}
\caption{Test suite organization. All 317 tests pass on CI.}
\label{tab:tests}
\end{table}

\paragraph{Comonad law tests.}
The left identity law ($\mathsf{extend}\;\mathsf{extract} = \mathrm{id}$)
is tested for zippers at multiple focus positions. The right identity law
($\mathsf{extract} \circ \mathsf{extend}\;f = f$) is tested at every
position in a 4-element zipper, confirming that the \texttt{Lefts}/\texttt{Rights}
iterators correctly enumerate all positions. The Writer comonad
(\texttt{WriterZipper}) is additionally verified for all three comonad
laws with monoid accumulation, monoid laws for
$\mathsf{DeletionSet}$ (identity, associativity, idempotence), and
equivalence between the Writer pipeline and the sentinel-based pipeline
on all 11 gradation patterns.

\subsection{L3$'$ Proof for the Writer Comonad}
\label{app:l3-proof}

We must show
$\mathsf{extend}_W\;g \circ \mathsf{extend}_W\;f
  = \mathsf{extend}_W\;(g \circ \mathsf{extend}_W\;f)$.
Expanding the left side, the first $\mathsf{extend}_W$ applies $f$ at
every position, accumulating a deletion set $D_f = \bigcup_i \pi_1(f(w, z_i))$;
the second applies $g$ to the result, accumulating
$D_g = \bigcup_j \pi_1(g(w \cup D_f, z'_j))$.
The total accumulated set is $(w \cup D_f) \cup D_g$.
On the right side, $\mathsf{extend}_W$ applies the composed arrow
$g \circ \mathsf{extend}_W\;f$ at each position, feeding $f$'s output
(including its accumulated deletions) directly into $g$, yielding
$w \cup (D_f \cup D_g)$.
Because $(\mathit{DS}, \cup, \emptyset)$ is an associative monoid,
$(w \cup D_f) \cup D_g = w \cup (D_f \cup D_g)$,
so the deletion components coincide.
The character components coincide by L3$'$ of the underlying Zipper
comonad.

\paragraph{Roundtrip tests.}
For non-destructive gradation patterns (clusters and single consonants),
the roundtrip property $\mathsf{strong}(\mathsf{weak}(\mathit{word}))
= \mathit{word}$ is verified. For example, weakening \emph{kampa} to
\emph{kamma} and then strengthening back produces \emph{kampa}. The
geminate patterns do not satisfy roundtrip because weakening
\emph{kaappi} to \emph{kaapi} (via $\mathit{DeletionSet}$ materialization)
loses the information that the original contained a geminate;
strengthening \emph{kaapi} cannot distinguish between a weakened
geminate and a lexical singleton. This asymmetry is a fundamental
property of the deletion operation, not a limitation of the comonadic
framework.

\section{Extended Evaluation Details}
\label{app:evaluation}

\subsection{Vowel Harmony Verification}
\label{app:eval-harmony}

Vowel harmony resolution is tested for back-vowel stems
(\emph{talo} + \emph{-ssA} $\to$ \emph{talossa}), front-vowel stems
(\emph{p\"oyd\"a} + \emph{-ssA} $\to$ \emph{p\"oyd\"ass\"a}), and
stems with only neutral vowels (\emph{tie} + \emph{-ssA} $\to$
\emph{tiess\"a}, defaulting to front). The $\mathsf{detect\_harmony}$
function correctly scans leftward through neutral vowels to find the
nearest non-neutral vowel. All three archiphonemes (A, O, U) are tested
in both back and front contexts. The harmony function is idempotent:
applying it to already-resolved text produces no change.

\subsection{Pipeline Integration Examples}
\label{app:eval-pipeline}

The morphophonological pipeline
$\mathsf{gradation} \mathbin{>\!\!=\!\!>} \mathsf{harmony}
\mathbin{>\!\!=\!\!>} \mathsf{possessive}$
is tested on words requiring multiple simultaneous transformations:

\begin{itemize}
  \item \emph{rantAssA} (weak) $\to$ \emph{rannassa}
    (gradation: nt$\to$nn, harmony: A$\to$a)
  \item \emph{pukussA} (weak) $\to$ \emph{puussa}
    (gradation: k deleted, harmony: A$\to$a)
  \item \emph{kampAstAVn} (weak) $\to$ \emph{kammastaan}
    (all three rules active)
  \item \emph{kenk\"asstAVn} (weak, front) $\to$ \emph{keng\"ast\"a\"an}
    (gradation: nk$\to$ng, harmony: A$\to$\"a, possessive: V$\to$\"a)
\end{itemize}

\subsection{CG-lite Disambiguation Examples}
\label{app:eval-cg-examples}

The CG-lite module is evaluated on both constructed and corpus-derived
Finnish examples. In the rules below, Finnish CG tag names are used:
\emph{lukusana} = numeral, \emph{nimisana} = noun,
\emph{teonsana} = verb, \emph{laatusana} = adjective,
\emph{seikkasana} = adverb.

\paragraph{``kuusi koiraa'' (`six dogs').}
The word \emph{kuusi} has readings \{numeral/kuusi, noun/kuusi\}. The
rule \texttt{SELECT lukusana IF (+1 nimisana)} correctly selects the
numeral reading when followed by the noun \emph{koiraa}.

\paragraph{``kuusi kasvaa'' (`spruce grows').}
With the rule \texttt{SELECT nimisana IF (+1 teonsana)}, the noun reading
of \emph{kuusi} is correctly selected when followed by the verb
\emph{kasvaa}.

\paragraph{``ei voi'' (`cannot').}
The word \emph{voi} is ambiguous between `butter' (noun) and `can'
(verb). The rule \texttt{SELECT teonsana IF (-1 BASEFORM=ei)} correctly
selects the verb reading when preceded by the negation verb \emph{ei}.

\paragraph{3-rule cascade.}
A sentence position with 4-way ambiguity \{noun, verb, adjective,
adverb\} is reduced to a single reading through three sequential
$\mathsf{extend}$ passes:
(1)~\texttt{REMOVE laatusana IF (NOT -1 lukusana)},
(2)~\texttt{REMOVE seikkasana IF (-1 nimisana)},
(3)~\texttt{SELECT teonsana IF (+1 teonsana)}.
This demonstrates the progressive disambiguation enabled by coKleisli
composition, where each rule's pruning enables subsequent rules to make
sharper decisions.

\subsection{Qualitative Coverage Details}
\label{app:eval-qualitative}

\begin{table}[!t]
\centering
\small
\begin{tabular}{@{}llllc@{}}
\toprule
Pattern & In & Out & RT \\
\midrule
pp $\to$ p  & kaappi & kaapi  & --  \\
tt $\to$ t  & matto  & mato   & --  \\
kk $\to$ k  & kukka  & kuka   & --  \\
p $\to$ v   & tupa   & tuva   & \checkmark \\
t $\to$ d   & katu   & kadu   & \checkmark \\
k $\to$ $\emptyset$ & puku & puu & --  \\
mp $\to$ mm & kampa  & kamma  & \checkmark \\
lt $\to$ ll & kulta  & kulla  & \checkmark \\
nt $\to$ nn & ranta  & ranna  & \checkmark \\
rt $\to$ rr & parta  & parra  & \checkmark \\
nk $\to$ ng & kenk\"a & keng\"a & \checkmark \\
\bottomrule
\end{tabular}
\caption{Gradation patterns as coKleisli arrows (weakening
  direction). RT = roundtrip
  $\mathsf{strong}(\mathsf{weak}(w)) \!=\! w$.}
\label{tab:gradation-results}
\end{table}

The non-roundtrip cases (geminate weakening and \emph{k} deletion)
involve information loss: the deletion of a character removes the
evidence needed for reverse reconstruction. Of the 11 patterns, 7
exhibit full bidirectional roundtrip correctness; the 4 non-roundtrip
cases all involve deletion.

\end{document}